\pdfoutput=1

\documentclass[11pt]{article}

\usepackage[preprint]{acl}
\usepackage{booktabs}

\usepackage{latexsym}
\usepackage{amsmath}
\usepackage{paralist} 
\usepackage{enumitem}
\usepackage{tcolorbox}
\usepackage{placeins}
\usepackage{makecell}
\usepackage{pifont}  

\usepackage{colortbl}

\usepackage{listings}
\usepackage[subtle]{savetrees}

\usepackage{amsthm}
\usepackage{cleveref}

\newtheorem{theorem}{Theorem}[section]
\newtheorem{lemma}[theorem]{Lemma}
\newtheorem{proposition}[theorem]{Proposition}
\usepackage{caption}
\usepackage{tabularx}
\usepackage{mathtools}
\DeclarePairedDelimiter\abs{\lvert}{\rvert}

\newcommand{\sinceC}[1]{\textbf{\#}\!\left[#1\right]}
\usepackage{siunitx}
\usepackage{tikz}

\usepackage[T1]{fontenc}

\usepackage[utf8]{inputenc}

\usepackage{microtype}

\usepackage{times}
\usepackage{newtxmath}
\usepackage{inconsolata}

\usepackage{float}

\usepackage{graphicx}
\usepackage{multirow}
\usepackage{multicol}
\usepackage{amssymb}

\newcommand{\Dppt}{\mathcal{D}_{\text{ppt}}}
\newcommand{\Dpt}{\mathcal{D}_{\text{pt}}}

\newcommand\CRASP{\textsf{C-RASP}}

\newcommand\KtSharp{\ensuremath{\mathsf{K}_t[\mathbf{\#}]}}

\newcommand\FOM{\textsf{FO(M)}}

\title{Between Circuits and Chomsky: \\ Pre-pretraining on Formal Languages Imparts Linguistic Biases}

\author{
    Michael Y. Hu$^{1}$ \quad Jackson Petty$^{2}$ \quad Chuan Shi$^{1}$ \quad William Merrill$^{1}$ \quad Tal Linzen$^{1,2}$ \\ \\
    $^{1}$Center for Data Science \quad $^{2}$Department of Linguistics \\
    New York University \\
    \texttt{\{michael.hu, petty, cs5526, willm, linzen\}@nyu.edu}
}

\begin{document}
\maketitle
\begin{abstract}

Pretraining language models on formal language can improve their acquisition of natural language. Which features of the formal language impart an inductive bias that leads to effective transfer?
Drawing on insights from linguistics and complexity theory, we hypothesize that effective transfer occurs when two conditions are met: the formal language should capture the dependency structures present in natural language, and it should remain within the computational limitations of the model architecture. 
We experiment with pre-pretraining (training on formal language before natural languages) on transformers and find that formal languages capturing hierarchical dependencies indeed enable language models to achieve lower loss on natural language and better linguistic generalization compared to other formal languages. 
We also find modest support for the hypothesis that the formal language should  fall within the computational limitations of the architecture.
Strikingly, pre-pretraining reduces loss more efficiently than training on a matched amount of natural language.
For a 1B-parameter language model trained on roughly 1.6B tokens of natural language, pre-pretraining achieves the same loss and better linguistic generalization with a 33\% smaller token budget. 
Finally, we also give mechanistic evidence of transfer from formal to natural language: attention heads acquired during pre-pretraining remain crucial for the model's performance on syntactic evaluations.\footnote{Code is available at \url{https://github.com/michahu/pre-pretraining}.}
%

\end{abstract}

\section{Introduction}

Language models have achieved impressive performance on many tasks, but they remain data-hungry, requiring five to six orders of magnitude more data than humans to achieve human-level performance \cite{warstadt-etal-2023-findings,Gilkerson2017MappingTE}. This high data requirement presents challenges for training models in low-resource settings \cite{Zhong2024OpportunitiesAC,Hettiarachchi2024OverviewOT}, understanding how language models can serve as cognitive models of language acquisition with human-like data constraints \cite{big-wilcox}, and continuing to improve models even after most of the existing natural language data has been used for pretraining \cite{Villalobos2022WillWR}. Thus, data efficiency during training is an important frontier for language models. 

\begin{figure}[t]
    \begin{tabularx}{0.95\columnwidth}{X||c|c}
    & Context-free & Context-sensitive \\ 
    \hline\hline
    \multirow{2}{*}{$\CRASP$} & \multirow{2}{*}{1-Dyck} & $k$-Shuffle \cellcolor[rgb]{0.937,0.922,0.99} \\ 
    & & Dyck \cellcolor[rgb]{0.937,0.922,0.99} \\
    \hline
    \multirow{2}{*}{$\FOM$}  & \multirow{2}{*}{$k$-Dyck} & \multirow{2}{*}{$ww$} \\ 
    & & \\ 
    \end{tabularx}
    \raggedleft
    \includegraphics[width=\columnwidth]{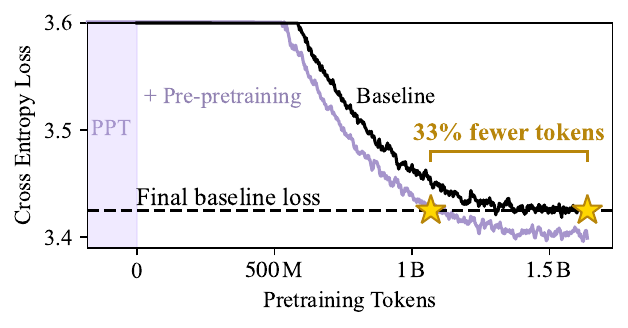}
    \caption{The intersection of Chomsky and circuit hierarchies (top), where $\CRASP \subset \FOM$ and context-free $\subset$ context-sensitive. Within this $2\times2$, we find that pre-pretraining on $k$-Shuffle Dyck, a context-sensitive language definable in $\CRASP$, lets 1B-parameter models match the final baseline performance of no pre-pretraining with 33\% fewer training tokens (bottom). See \S \ref{sec:methods-grammars}.}
    \label{fig:conceptual}
\end{figure}

A recently-explored approach for increasing data efficiency teaches models useful inductive biases by first training them on formal languages before training on natural language \cite{papadimitriou-jurafsky-2020-learning,Chiang2021OnTT,mccoy2023modelingrapidlanguagelearning}. We refer to this paradigm as pre-pretraining. What features of formal languages make transfer to natural language effective?
\citet{papadimitriou-jurafsky-2023-injecting} show that within the Chomsky hierarchy, context-sensitive languages transfer best to natural language compared to simpler classes of languages. We expand on their investigation and explore an additional factor: the computational limitations of the language model's architecture. In particular, transformers---the architecture that underlies most popular language models---cannot learn all context-sensitive languages, both in theory and practice \cite{strobl-etal-2024-formal,merrill2023a}. 
In fact, within all levels of the Chomsky hierarchy, some languages are harder for transformers to learn than others, and many are impossible for them to learn \cite{merrill2023talecircuitsgrokkingcompetition,merrill2024the}. 
Can a formal language give rise to positive transfer even when it cannot be fully learned by a transformer?

In this work, we hypothesize that optimal transfer from formal to natural language in transformer language models occurs at the intersection of two theoretical hierarchies: the Chomsky hierarchy of formal languages and the circuit complexity hierarchy that bounds transformer computational power (see~\S \ref{sec:methods}). Specifically, we hypothesize that effective pre-pretraining languages should be:
\vspace{0.25ex}
\begin{compactenum}
    \item expressive enough to capture hierarchical natural language dependencies, and
    \item learnable by transformers in a way that generalizes to longer strings than observed in training. 
\end{compactenum} 
\vspace{0.25ex}
To satisfy the second condition, we define our formal languages in $\CRASP$ \cite{yang2024counting}, a restricted programming language whose functions allow transformers to exhibit length generalization \cite{huang2025a}.

Our empirical results support the first part of the hypothesis and provide some support for the second part~(\S \ref{sec:expressivity}). Pre-pretraining on languages with hierarchical dependencies outperforms pre-pretraining on any of the other formal languages that we tested---in fact, it outperforms pre-pretraining on a matched amount of natural language. Of the formal languages with hierarchical dependencies, those that are definable in $\CRASP$ generally achieve equal or better performance, but they are only clearly superior on some of the tasks we evaluated.

Next, we show that when positive transfer occurs, the model reuses attention heads it learned during pre-pretraining, suggesting that mechanisms from pre-pretraining transfer to natural language (\S \ref{sec:pruning}). Finally, we scale up our experiments to a 1B-parameter language model, and show that in pre-pretraining is effective in that size as well, increasing token efficiency by 33\% (\S \ref{sec:analysis}). 
Overall, we conclude that formal language pre-pretraining is an effective way to improve generalization and data efficiency, and propose a hypothesis for the particular formal languages that are most promising for this purpose.




\section{Background}
\label{sec:background}


\subsection{The Chomsky Hierarchy}

The Chomsky hierarchy \cite{chomsky-hierarchy} is a nested classification of increasingly-complex formal languages. 
This classification is based on the kinds of computations needed to process formal structures resembling those found in human language. For example, regular languages, the least complex, can be recognized by finite-state automata. While regular languages can capture most phenomena in natural language phonology and morphology, they are insufficient for syntax: representing the hierarchical structure of natural language syntax with a finite-state automaton would require infinitely many states \cite{chomsky-three}. Subsequent works showed that modeling some syntactic phenomena requires not only context-free but also context-sensitive grammars \cite{Shieber1985EvidenceAT}, though the prevalence of such phenomena may be limited.

\paragraph{Dyck languages.} A classic context-free language is $k$-Dyck: the language of well-balanced parentheses with $k$ bracket types. For example, \texttt{([])[]} is a valid $2$-Dyck string, where rounded and square parentheses are the two bracket types. $k$-Dyck is often taken as a canonical example of context-free hierarchical structure because any context-free language can be reduced to Dyck via a single transformation (inverse homomorphism) and intersection with a regular language \citep{chomsky-1959-algebraic}.

\paragraph{Shuffle Dyck.} Removing the constraint that Dyck braces must be well-nested, but maintaining the constraint that every opening brace must be closed and vice versa, yields $k$-Shuffle Dyck,\footnote{Despite what its name might suggest, $k$-Shuffle Dyck does not randomly shuffle strings from $k$-Dyck. Every opening brace in $k$-Shuffle Dyck must still be closed by a matching closing brace later in the string; this constraint would not in general be satisfied by randomly shuffled $k$-Dyck strings. Instead, $k$-Shuffle Dyck can be defined by interleaving $k$ 1-Dyck strings with different braces \citep{suzgun-etal-2019-lstm}, as if by riffle shuffling. We use the terminology ``Shuffle Dyck" for consistency with prior work.} a minimal relaxation of $k$-Dyck that is strictly context-sensitive rather than context-free \cite{suzgun-etal-2019-lstm,strobl-etal-2024-formal}.
Crossing braces in $k$-Shuffle Dyck can be thought of as a formal model of the cross-serial dependencies underlying aspects of language argued to be context-sensitive \citep{papadimitriou-jurafsky-2023-injecting}.

\subsection{The Circuit Hierarchy}
\label{sec:circuit-background}

We focus in this work on transformer language models. There are languages at each level of the Chomsky hierarchy that a transformer cannot recognize \citep{merrill2023a,liu-etal-2024-shortcuts,strobl-etal-2024-formal}.
Thus, the Chomsky hierarchy alone does not precisely capture how difficult a language is for transformers to learn: for instance,  transformers can learn some context-free languages \cite{butoi2025training} and yet fail to learn other regular languages \cite{merrill2024the}.
To better understand the expressive power of transformers, recent work has analyzed formal languages within a different hierarchy: the circuit complexity hierarchy, which better captures the computations performed by transformers \cite{Hao2022FormalLR,yang2024masked}. Here, we will focus on two logics that emerge from the circuit complexity viewpoint: $\FOM$ \cite{merrill2023a} and $\CRASP$ \cite{yang2024counting}.

\paragraph{$\FOM$.} First-order logic with majority, or $\FOM$, is a provable \emph{upper bound} on the languages that transformers can express: that is, any transformer that recognizes a language can be converted into an $\FOM$ program that defines (or recognizes) the same language \citep{merrill2023a}.
$\FOM$ programs operate by computing counts over the number of indices in an input string that satisfy certain predicates.
For example, $Q_a(i)$ is a basic predicate that checks whether input token $i$ is an $a$. The following $\FOM$ program uses $Q_a(i)$ to define the language of strings with exactly 3 $a$'s:
\begin{equation} \label{eq:count3}
    \#i \leq n [Q_a(i)] \; = \; 3
\end{equation}
Beyond this example, $\FOM$ can implement a rich variety of programs by nesting quantifiers and building complex predicates out of logical ($\wedge, \vee, \neg$) and arithmetic operators ($+, =, <$).
In particular, $\FOM$ can define the $k$-Dyck language for any $k \geq 1$ (Proposition~\ref{prop:kdyck} in the Appendix).
For example, the following program defines $1$-Dyck:
\begin{align}
    \mathsf{depth}(i) \;\; &\equiv \;\; \# j \leq i [Q_((i)] - \# j \leq i [Q_)(i)] \nonumber \\
    [\mathsf{depth}(n) &= 0] \; \wedge \; \# i \leq n [ \mathsf{depth}(i) < 0] = 0 \label{eq:1dyck}
\end{align}
To define $2$-Dyck, this can be extended by modifying $\mathsf{depth}$ to track two bracket types and computing the following depth index:
\begin{align}
    \mathsf{dindex}(i) \;\; \equiv \;\; \#j \leq i [\mathsf{depth}(i) = \mathsf{depth}(j) ] \label{eq:dindex}
\end{align}
To finish the definition, we add a condition to enforce that any open and close brace paired by $\mathsf{depth}$ and $\mathsf{dindex}$ also match in their type (i.e., both are parentheses or both are square braces).
See Proposition \ref{prop:kdyck-fom} for further details.

\paragraph{$\CRASP$.} While any transformer can be compiled into $\FOM$, it is not necessarily the case that any $\FOM$ program can be implemented by a transformer.
$\CRASP$ is a restriction of $\FOM$ designed to be a \emph{lower bound} on what transformers can express: that is, if a language is definable in $\CRASP$, then there exists a transformer that recognizes it \citep{yang2024counting}.\footnote{$\CRASP$ is a well-defined variant of the Restricted Access Sequence Processing programming language \citep[RASP;][]{weiss21rasp,lindner2023tracr}.}
The most crucial restriction for our purposes is that each $\CRASP$ predicate can only refer to one index variable $i$, whereas in $\FOM$ predicates can refer to two (or more) indices $i, j$ introduced by different quantifiers (for more detail, see \citealp{yang2024counting}).
This means $\CRASP$ can define \eqref{eq:count3} or \eqref{eq:1dyck} above, but not $k$-Dyck for $k \geq 2$, as $\CRASP$ cannot express the function $\mathsf{dindex}$ in \eqref{eq:dindex}, which compares the depth of two different indices.

Recent work has also suggested a connection between $\CRASP$ and transformers' ability to generalize to strings longer than those observed in training~\citep{zhou2024what,huang2025a}: the definability of a language $L$ in $\CRASP$ predicts whether transformers can reliably length-generalize when trained on strings from $L$.
One interpretation of this finding is that mechanisms expressible in $\CRASP$ may be more robustly learnable by transformers.
We thus hypothesize that we will observe more reliable transfer from pre-pretraining transformers on formal languages that can be defined in $\CRASP$ compared to languages that cannot.

\section{Methods}
\label{sec:methods}


\subsection{Defining Pre-pretraining}
\label{sec:methods-ppt}

We train a language model using an optimizer $\mathcal{A}(\mathcal{D}, t, \theta_{\text{init}})$ which returns parameters $\theta_t$ after $t$ timesteps (gradient updates). We apply $\mathcal{A}$ sequentially:
\vspace{0.25ex}
\begin{compactenum}
    \item Pre-pretrain for $t_0$ steps on dataset $\Dppt$ to obtain model parameters $\theta_{t_0}$.
    \item Pretrain for $t_1$ steps on dataset $\Dpt$ to obtain $\theta_{t_1}$.
\end{compactenum}
\vspace{0.25ex}

Our objective is to minimize the expected loss on the pretraining dataset, i.e. to find
$\arg \min_{\theta_{t_1}} \;\; \mathbb{E} [\ell(\Dpt, \theta_{t_1})]$. We hold $\mathcal{A}$'s hyperparameters, $t_1$, and $\Dpt$ fixed, and we transfer model parameters directly from pre-pretraining to pretraining. In other words, to minimize $\ell(\Dpt, \theta_{t_1})$, we can only change the pre-pretraining dataset $\Dppt$ and duration $t_0$. 
We compare pre-pretraining on our proposed $\Dppt$ datasets (\S \ref{sec:methods-grammars}) against several baselines:

\vspace{0.25ex}
\begin{compactitem}
    \item No pre-pretraining $(t_0 = 0)$.
    \item Pre-pretraining on random binary strings.
    \item Pre-pretraining on random strings of $k$ integers.
    \item Pre-pretraining on unseen natural language data $\Dpt^{*}$ drawn from the same distribution as $\Dpt$.
\end{compactitem}
\vspace{0.25ex}
Aside from the no-pre-pretraining baseline, we pre-pretrained the baselines for $t_0=500$ steps, the optimal number of steps for $k$-Shuffle Dyck (see~\S\ref{sec:expressivity}).
We note that the natural language pre-pretraining baseline is not equivalent to training on the union of $\Dpt^{*}$ and $\Dpt$ for longer, since pre-pretraining on natural language uses learning rate warmup twice, once in pre-pretraining and once in pretraining. 

Lower validation loss compared to the no-pre-pretraining baseline would indicate that pre-pretraining on formal languages is beneficial. The random string baselines help establish whether this effect is specific to the particular formal languages we study. Finally, outperforming pre-pretraining on $\Dpt^{*}$ would suggest that formal languages provide a better initialization for pretraining than the pretraining data itself.

\paragraph{Evaluation.} In addition to measuring validation loss, we perform targeted evaluations for grammaticality and verbatim retrieval. For grammaticality judgments, we compare the likelihood assigned by the model to minimal pairs of sentences that differ only in their grammaticality (e.g., \textit{Only Bill would ever complain} is grammatical, but \textit{Even Bill would ever complain} is not). Accuracy is measured as the proportion of pairs where the grammatical sentence is assigned higher likelihood than the ungrammatical one \cite{marvin-linzen-2018-targeted}. We use the BLiMP grammaticality judgment dataset \cite{warstadt-etal-2020-blimp-benchmark}. 
Verbatim retrieval tests language modeling on text passages with repeated lists \cite{armeni-etal-2022-characterizing,armeni-etal-2024-transformer}; the model is expected to assign a very high likelihood to the words in the second repetition of the list, such that lower loss indicates better performance. Both evaluations assess models' ability to learn and apply consistent patterns---a capability that could benefit from pre-pretraining on formal languages might strengthen. 
For examples of these evaluations, see Tables~\ref{tab:blimp} and~\ref{tab:verbatim} in the Appendix.

\paragraph{Efficiency.} In the regime with plentiful pretraining data, an ideal pre-pretraining language should minimize the number of pre-pretraining steps $t_0$ required: if a formal language requires very large $t_0$ for effective transfer, then simply pretraining on natural language, without any pre-pretraining, would be more practical in terms of total compute (though even in this case pre-pretraining may still be beneficial when the amount of data available for pretraining is small, for example in low-resource languages). We quantify efficiency using the marginal rate of substitution (MRS) between formal and natural language at 10,000 steps of natural language pretraining. In other words, we ask: if we train on 500 steps of the formal language, how many more steps does it take for the natural-language-only baseline to catch up?  

For example, let $x$ be the number of pre-pretraining steps and $y$ be the number of pretraining steps, and suppose the following two pairs $(x, y)$ of training steps achieve the same final loss: $(0,~10,000)$ and $(500,~6,000)$. Then the marginal rate of substitution is 
$$\frac{|y_1-y_2|}{|x_1-x_2|} = \frac{|10,000 - 6,000|}{|0 - 500|} = 8.$$
The gain in token efficiency would be 
$$1 - \frac{6,000 + 500}{1,0000} = 35\%$$
For a visualization, see Figure \ref{fig:mrs}.

In our setting, a good pre-pretraining language would (1) minimize the amount of pre-pretraining steps $t_0$ (\textbf{efficiency}), and (2) increase the evaluation \textbf{performance} of the language model. 

\subsection{Between Circuits and Chomsky}
\label{sec:methods-grammars}

We hypothesize that a good pre-pretraining language should both mimic particular aspects of the complexity of natural language and be robustly learnable by transformers in a way that generalizes to longer strings than observed in training.
Because, as we discussed in \S\ref{sec:background}, natural language is hierarchically structured and $\CRASP$ is a formal model of what transformers can learn robustly, this motivates the following hypothesis: 

\begin{tcolorbox}[colback=gray!10,colframe=gray!70]
\textbf{Expressivity hypothesis:} A formal language that confers a helpful inductive bias should be hierarchically structured (either context-free or context-sensitive) and definable in $\CRASP$.
\end{tcolorbox}

To test this hypothesis, we pre-pretrain transformer language models on the following formal languages:
\vspace{0.25ex}
\begin{compactenum}
    \item 1-Dyck: the nested parentheses language. This language is context-free and in $\CRASP$.
    \item $k$-Dyck: contains $k$ different types of parentheses. The language is context-free and in $\FOM$ but not in $\CRASP$.
    \item $k$-Shuffle Dyck: $k$-Dyck with cross-serial dependencies. This language is context-sensitive and in $\CRASP$.\footnote{Figure \ref{code:shuff-dyck} shows a minimal code snippet for $\CRASP$.}
    \item $ww$: The copy language. This language is context-sensitive and in $\FOM$ but not in $\CRASP$.
\end{compactenum}
\vspace{0.25ex}

\begin{table}[htb]
\centering
\begin{tabularx}{0.75\columnwidth}{Xc}
\toprule
\textbf{Language}  & \textbf{Example} \\
\midrule
1-Dyck & \texttt{(\,(\,(\,)\,)\,)} \\
$k$-Dyck & \texttt{(\,\textcolor{blue}{[}\,\textcolor{red}{\{}\,\textcolor{red}{\}}\,\textcolor{blue}{]}\,)} \\
$k$-Shuffle Dyck  & \texttt{(\,\textcolor{blue}{[}\,\textcolor{red}{\{}\,\textcolor{blue}{]}\,)\,\textcolor{red}{\}}} \\ 
$ww$ & \texttt{1\,\textcolor{blue}{2}\,\textcolor{red}{3}\,1\,\textcolor{blue}{2}\,\textcolor{red}{3}} \\
\bottomrule
\end{tabularx}
\caption{Examples of our pre-pretraining languages. } 
\label{tab:stats} 
\end{table}

The three variants of Dyck languages model hierarchical structure, while $ww$ has a fixed dependency structure that maps the first half of the string onto the second half (Table \ref{tab:stats}). Proofs of where these languages lie on the Chomsky and circuit hierarchies can be found in Appendix \ref{app:proofs}.\footnote{The languages $\textsf{NEST}$ and $\textsf{CROSS}$ from \citet{papadimitriou-jurafsky-2023-injecting} are instances of $k$-Dyck and $k$-Shuffle Dyck, respectively. Their results align with our hypothesis.}

\begin{figure*}
    \centering
    \includegraphics[width=\linewidth]{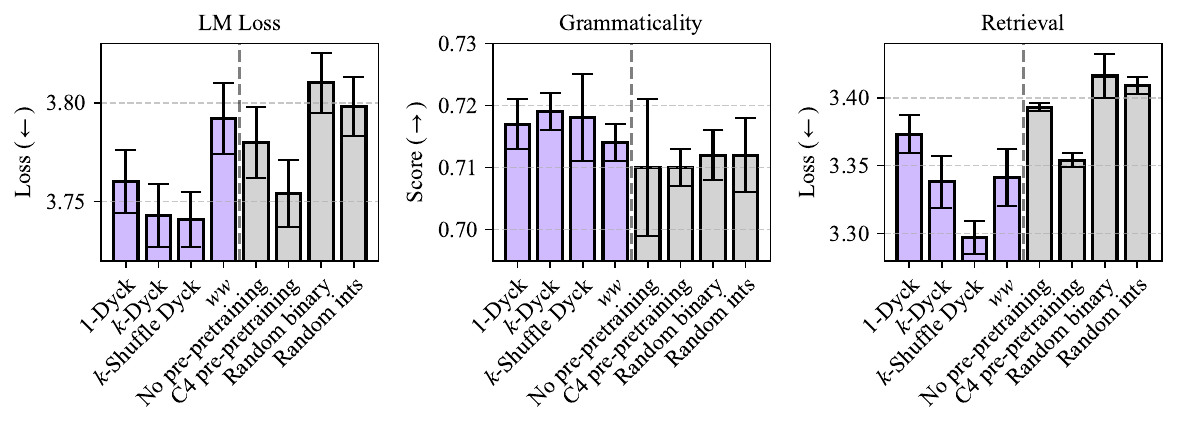}
    \caption{Evaluating models on overall language modeling loss, grammaticality and retrieval, at the optimal amount of pre-pretraining $t_0^*$ for each formal language (the value of $t_0^*$ for each language is determined based on Figure~\ref{fig:scaling}).}
    \label{fig:ppt-results}
    \vspace{-1em} 
\end{figure*}

We deliberately chose languages that are similar to each other. $k$-Dyck and $k$-Shuffle Dyck can be seen as different extensions of 1-Dyck: $k$-Dyck swaps out paired parentheses in valid 1-Dyck strings with new parentheses pairs, while $k$-Shuffle Dyck effectively interleaves several 1-Dyck sequences \cite{suzgun-etal-2019-lstm}. Finally, $ww$ contrasts with $k$-Shuffle Dyck as a maximally context-sensitive language, since \textit{all} the dependencies in $ww$ are cross-serial (i.e. none are nested within one another). 

We construct 1-Dyck, $k$-Dyck, and $k$-Shuffle Dyck corpora with matching depth distributions by randomly opening or closing parentheses with probability $p = 0.5$, which yields a harmonic distribution over depths. We truncate the length of the sequences at 2048. 
We also match the vocabulary size: $k$-Dyck, $k$-Shuffle Dyck, and $ww$ corpora each have 128 unique vocabulary items, or 64 unique parentheses pairs ($k=64$) for the Dyck languages (we explore the effect of this hyperparameter in \S\ref{sec:analysis}). All models are pre-pretrained on the same number of tokens with sequence packing. 

\section{Testing the Expressivity Hypothesis}
\label{sec:expressivity}

\begin{figure}[htbp] 
    \centering 
    \includegraphics[width=\columnwidth]{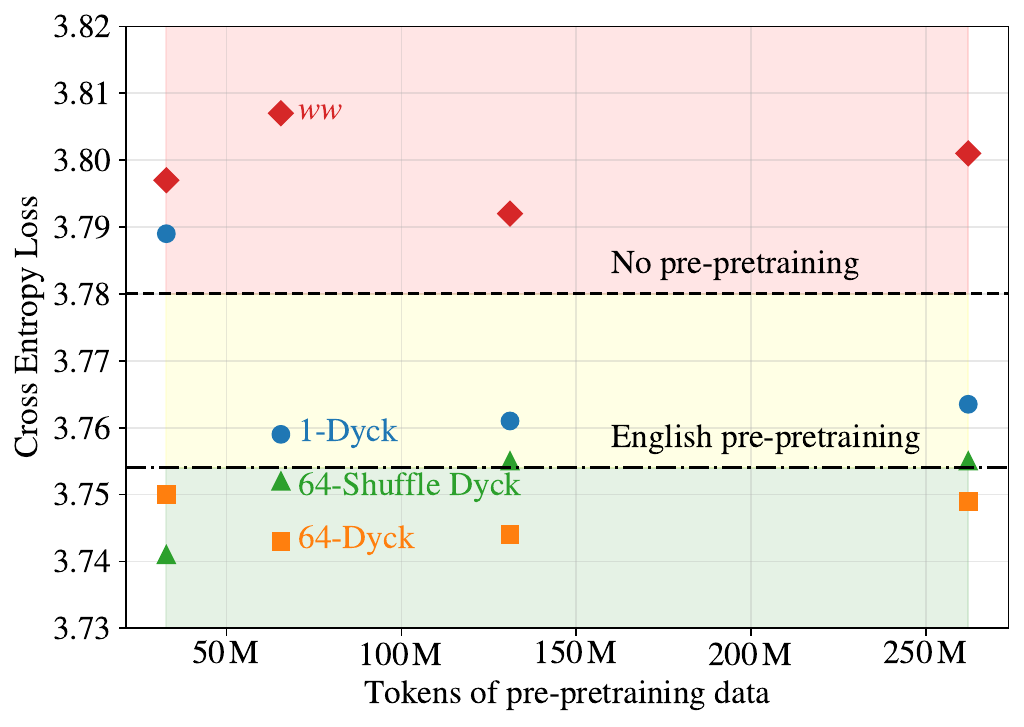} 
    \caption{C4 validation loss as a function of pre-pretraining tokens. For the formal languages that improve validation loss over no pre-pretraining, there is an optimal training duration after which additional pre-pretraining is harmful.} 
    \label{fig:scaling}
    \vspace{-1em}
\end{figure}

For natural language ($\Dpt$), we trained Pythia 160M models \cite{pythia} for 10,000 steps, or roughly 665 million tokens. We use C4 as the natural language dataset  \cite{Raffel2019ExploringTL}. For training hyperparameters, see Appendix~\ref{app:hyperparams}.

\paragraph{Efficiency.} We find that the optimal amount of pre-pretraining $t_0^*$ differs between formal languages. To estimate $t_0^*$, we sweep four pre-pretraining durations $t_0$. Figure \ref{fig:scaling} shows validation loss on natural language after pre-pretraining for 30 to 260 million tokens of formal language (500 to 4000 gradient updates). 

While both $k$-Shuffle Dyck and $k$-Dyck outperform natural language pre-pretraining, $k$-Shuffle Dyck is more efficient with $t_0^* = 500$ compared to $t_0^* = 1000$ for $k$-Dyck. Pre-pretraining  on $ww$  is unhelpful at all durations. For each of the languages where pre-pretraining is effective, there is an optimal duration after which additional formal language pre-pretraining leads to less effective transfer overall. $k$-Shuffle Dyck has the highest MRS, indicating that it replaces tokens on natural language most efficiently (see Table~\ref{tab:ppt-results} in the Appendix). Furthermore, the MRS for 1-Dyck, $k$-Dyck, and $k$-Shuffle Dyck are all greater than 1, indicating that exchanging natural language for these formal languages is \textbf{compute-optimal} in our setting.

\paragraph{Performance.} $k$-Shuffle Dyck is the best-performing formal language on the natural language validation set from C4, followed by $k$-Dyck (Figure \ref{fig:ppt-results}). Interestingly, pre-pretraining on all four formal languages improves accuracy in grammaticality, but pre-pretraining on natural language does not (for grammaticality accuracies by category, see Figure~\ref{fig:blimp-accuracies} in the Appendix). This indicates that formal language pre-pretraining also changes models' generalization properties, in addition to driving the language modeling loss lower. We hypothesize this is because pre-pretraining induces representations useful for modeling hierarchical structure; we revisit this point in \S \ref{sec:pruning}.

Pre-pretraining on either random binary strings or $k$-integer strings has a negative effect: it results in higher validation loss than no pre-pretraining. This rules out the hypothesis that any pre-pretraining is helpful, regardless of the data being pre-pretrained on. 

\paragraph{Summary.} Hierarchical dependencies, which both $k$-Dyck and $k$-Shuffle Dyck have, appear to be crucial for positive transfer from formal to natural language. Although of these two languages only $k$-Shuffle Dyck is expressible by $\CRASP$, it only significantly outperforms $k$-Dyck on verbatim retrieval. That being said, $k$-Shuffle Dyck is more efficient than $k$-Dyck, achieving its optimal amount of pre-pretraining 500 steps earlier. Taken together, we find modest evidence supporting the importance of the expressibility of the language in $\CRASP$.

\section{Mechanistic Analysis: Subnetworks}
\label{sec:pruning}

What is the mechanism by which pre-pretraining facilitates the learning of natural language? We hypothesize that the model implements next-token prediction on $\Dppt$ using a sparse subnetwork, or some subset of the total parameters $\mathcal{M}(\theta_{t_0}) \subset \theta_{t_0}$ ($\mathcal{M}$ for short). Once we transfer $\theta_{t_0}$ to learn $\Dpt$, this subnetwork $\mathcal{M}$ continues to drive the performance of language modeling on $\Dpt$.

\begin{tcolorbox}[colback=gray!10,colframe=gray!70]
\textbf{Subnetworks hypothesis:} Attention heads established during formal language pre-pretraining are later used to represent the hierarchical structure of natural language.
\end{tcolorbox}

We test this hypothesis by ablating attention heads of the pre-pretraining subnetwork and comparing the drop in performance against random attention head ablations. Concretely, we pre-pretrain on $\Dppt$ and prune the model to find the sparse subnetwork $\mathcal{M}(\theta_{t_0})$. We use the heuristic core pruning algorithm from \citet{bhaskar-etal-2024-heuristic}, which  iteratively removes attention heads from the transformer using structured pruning \citep{xia-etal-2022-structured} while minimizing the tradeoff between sparsity and language modeling loss on $\Dppt$. After transfer and training on $\Dpt$, we evaluate the masked model $\mathcal{M}(\theta_{t_1})$ against a model $\mathcal{M}_{\text{null}}(\theta_{t_1})$ where a subnetwork with the same number of randomly chosen attention heads was ablated.

Positive transfer from $\Dppt$ to $\Dpt$ could occur for reasons unrelated to subnetworks (e.g., computations are distributed across all heads or in other components of the model). In this case, the masked model $\mathcal{M}$ should perform no better than random masks $\mathcal{M}_{\text{null}}$ when applied to $\theta_{t_1}$. However, if pre-pretraining does induce useful inductive biases, we would expect $\mathcal{M}$ to be an important subnetwork even after training on $\Dpt$. So in the alternative hypothesis, $\mathcal{M}$ should significantly outperform $\mathcal{M}_{\text{null}}$.

\paragraph{Results.} After pre-pretraining on $k$-Shuffle Dyck, we ablate $50\%$ of the attention heads. Following previous work \cite{bhaskar-etal-2024-heuristic,zhang2024towards}, we replace an ablated head with its mean activation. We compare $\mathcal{M}$ to the random subnetwork $\mathcal{M}_{\text{null}}$ and to $\mathcal{M}$'s complement subnetwork $\mathcal{M}^c$.


\begin{figure}[htbp] 
    \centering 
    \includegraphics[width=\columnwidth]{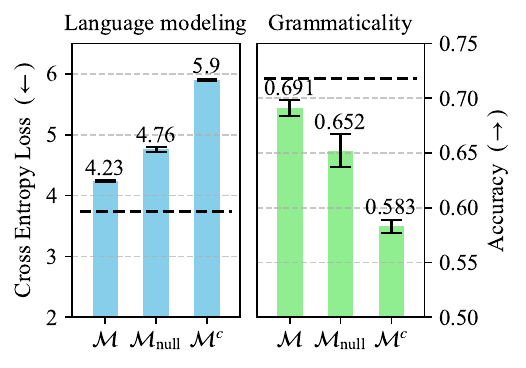} 
    \caption{Language modeling and grammaticality performance for the learned subnetwork $\mathcal{M}$, its complement $\mathcal{M}^c$, and randomly sampled masks $\mathcal{M}_{\text{null}}$. $\mathcal{M}$ outperforms $\mathcal{M}^c$ and $\mathcal{M}_{\text{null}}$, indicating that the subnetwork learned during pre-pretraining continues to play a critical role after training on natural language. Dashed lines indicate performance of the base model without pruning.} 
    \label{fig:pruning} 
\end{figure}


We find that $\mathcal{M}$  outperforms $\mathcal{M}_{\text{null}}$ and $\mathcal{M}^c$ in both language modeling and grammaticality (Figure~\ref{fig:pruning}). We reject the null hypothesis that the subnetwork $\mathcal{M}$ established during pre-pretraining has the same performance as a randomly sampled subnetwork ($p \ll 0.001)$. Further supporting the role of the heads in $\mathcal{M}$, we find that $\mathcal{M}^c$, which excludes all of the heads identified by the pruning procedure, performs much more poorly than $\mathcal{M}_{\text{null}}$, that only includes a random subset of them. That being said, while the performance $\mathcal{M}$ is close to the performance of the full network, it does not quite match it, indicating that attention heads outside of $\mathcal{M}$ also play a role in processing natural language. 

A breakdown of grammaticality judgment accuracy by grammatical phenomenon shows that only a handful of phenomena are unaffected by masking, some substantially (e.g., the accuracy on subject-verb agreement drops 12 percentage points; see \ref{fig:blimp-mask-delta} in the Appendix). These phenomena appear to be ones that are syntactically simple but diagnose sensitivity to word structure (morphology), e.g., the distinction between \textit{broke} and \textit{broken}; we hypothesize that this aspect of linguistic knowledge is less likely to be mechanistically related to the processing of dependencies in a formal language.

\section{Additional Analyses}
\label{sec:analysis}

\begin{figure*}[t]
    \centering
    \includegraphics[width=\linewidth]{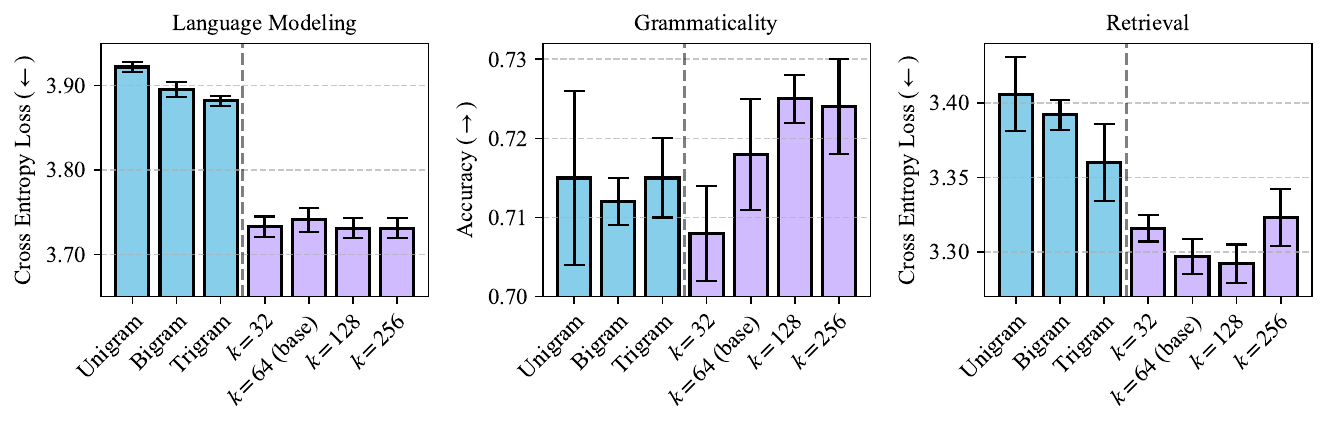}
    \caption{\textbf{Blue bars}: Testing if the benefit of pre-pretraining on $k$-Shuffle Dyck can be reduced to learning local statistics. We find that pre-pretraining on $n$-gram metamers \cite{Kumar2022DisentanglingAF} of $k$-Shuffle Dyck performs worse than pre-pretraining on $k$-Shuffle Dyck itself. \textbf{Purple bar}: Vocabulary size manipulation. The best vocabulary size for $k$-Shuffle Dyck is $k=128$.}
    \label{fig:ppt-ablations}
    \includegraphics[width=\linewidth]{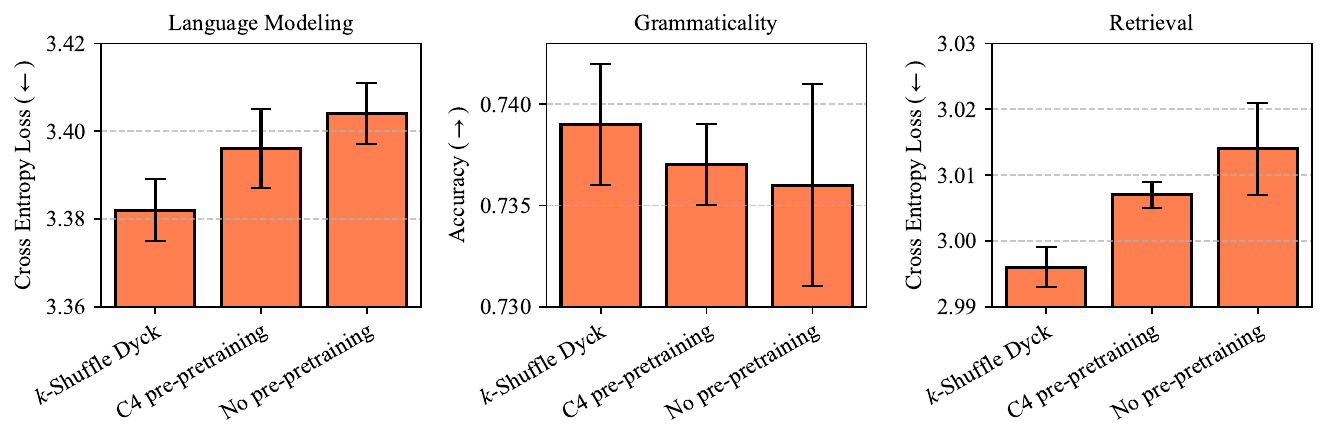}
    \caption{Pre-pretraining Pythia-1B on 1.6B tokens of $k$-Shuffle Dyck improves over the baselines, especially on language modeling and the retrieval evaluation.}
    \label{fig:ppt-1b}
\end{figure*}

This section reports three additional experiments. Due to the computational cost of pretraining, we focus on $k$-Shuffle Dyck, which performed well in our main experiments. First, we test and rule out the hypothesis that pre-pretraining on $k$-Shuffle Dyck is only effective because of its local statistical properties, and conclude that its effectiveness stems from its structural properties. Next, we study the impact of the vocabulary size hyperparameter $k$ on the effectiveness of transfer from $k$-Shuffle Dyck. Finally, we perform a larger scale training run with Pythia 1B and find that pre-pretraining on $k$-Shuffle Dyck helps in this setting as well. In all of these experiments, we used the optimal number of pre-pretraining gradient updates $t^*_0 = 500$ found in our main experiment (equivalent to 30 million tokens).  

\paragraph{Transfer is not only due to local statistical properties.} Could the successful transfer from $k$-Shuffle Dyck to natural language be due to the local statistical properties of  $k$-Shuffle Dyck, rather than its dependency structure? Learning local statistical regularities is consistent with the finding that neural networks can exhibit \textbf{distributional simplicity bias} (DSB)---they learn simpler statistical patterns, such as the mean and covariance of their representations, before progressing to higher-order relationships \cite{Saxe2013ExactST,lecun-eigen}; in particular, transformer language models learn $n$-gram statistics in order of increasing complexity \cite{belrose2024neural}.

To test this hypothesis, we create variants of $k$-Shuffle Dyck that share its local statistics but not its global, rule-based structure. Concretely, we train unigram, bigram, and trigram models on the pre-pretraining corpus we generated from $k$-Shuffle Dyck for our main experiment, and, using these $n$-gram models, we generate ``metamer datasets'' \cite{Kumar2022DisentanglingAF} of equivalent size. 

We find that pre-pretraining on metamer datasets is strictly less effective than pre-pretraining on $k$-Shuffle Dyck, ruling out the hypothesis that the benefit of pre-pretraining on $k$-Shuffle Dyck is due to local statistics. That being said, pre-pretraining on the unigram metamer performs the worst, followed by bigram and trigram, suggesting that local statistics may explain part of the success of pre-pretraining on structured languages. 

\paragraph{Larger vocabulary size may be beneficial.} To check whether better hyperparameters exist for $k$-Shuffle Dyck, we sweep its vocabulary size, trying $k=32,128$ and $256$ in addition to our previous experiments with $k=64$. We find that $k=128$ has the best performance across all metrics instead of $k=64$, suggesting there likely do exist better hyperparameters. 

Finding good ways to optimize these hyperparameters is an interesting area for future work. The hyperparameter tuning process for pre-pretraining is expensive, as evaluating the hyperparameters requires pretraining a language model. Nevertheless, various approximations such as early truncation exist in the hyperparameter tuning literature \cite{hyperband,Swersky2013MultiTaskBO}, and one can also use scaling laws to experiment at a smaller scale first \cite{Yang2022TensorPV}. 

\paragraph{Pre-pretraining is effective at the 1B scale too.} Finally, we examine whether our results generalize to larger settings by training Pythia-1B on 1.63B tokens from C4 (25,000 steps). In this setting, pre-pretraining on $k$-Shuffle Dyck continues to outperform the baselines on all evaluation metrics (Figure \ref{fig:ppt-1b}) and achieves the final loss of the no-pre-pretraining baseline after training for only 1.10B total tokens. This equates to a token efficiency gain of 33\% $\left(1 - \frac{1.10\text{B}}{1.63\text{B}}\right)$, or an MRS of 17.3 $\left(\frac{1.63\text{B}-1.10\text{B}}{0.03\text{B}}\right)$. \mbox{$k$-Shuffle} Dyck's MRS is $\gg$1 for both 160M and 1B training runs, suggesting that pre-pretraining could increase the efficiency of large-scale pretraining as well.

\section{Related Work}
 
The goal of pre-pretraining is similar to that of optimization-based meta-learning, which aims to create a weight initialization that allows the model to rapidly learn new tasks \cite{maml,fomaml} and languages \cite{mccoy2020universal,mccoy2023modelingrapidlanguagelearning}. The beneficial effect of pre-pretraining on formal language is consistent with the evidence of transfer from source code to natural language, especially for structured tasks \cite{petty2025how,aryabumi2025to,kim2024codepretrainingimprovesentity}. In the vision domain, \citet{synthetic-image-training} show that a thousand synthetically generated images can replace a million images from ImageNet-1k, in a similar spirit to our work. 

Transfer in NLP has also been studied across different languages and domains \citep{ruder-etal-2019-transfer,pruksachatkun-etal-2020-intermediate,deshpande-etal-2022-bert}. Most relevant to our work, \citet{mueller-linzen-2023-plant} show that pretraining on child-directed speech gives a better inductive bias for learning hierarchical syntactic features than standard pretraining corpora. Furthermore, introducing a small amount of synthetic, disambiguating data into pretraining can induce a language model to change its generalization strategy \cite{warstadt-etal-2020-learning}. Related but distinct from our approach are studies that use synthetic data sampled from a formal language to evaluate models' generalization behavior in a controlled way \cite{mccoy-etal-2019-right,kim-linzen-2020-cogs,li-etal-2023-slog}.

Pre-pretraining is a form of curriculum learning \citep{bengio2009curriculum}, the approach of actively adjusting properties of the data during training. Recent work has developed algorithms that automate the discovery of language modeling curricula \cite{chen2025aioli,jiang2025adaptive}, and many language model training recipes introduce different data mixtures at different stages  \cite{smollm2,2olmo2furious,instructgpt}. The positive results of our experiments contrast with the largely negative results of the attempts to improve language models' data efficiency via linguistics-inspired curriculum learning, as part of the BabyLM challenge \cite{warstadt-etal-2023-findings,hu-etal-2024-findings}, pointing to the crucial effect of the particular data presented as part of the curriculum.

\section{Discussion}

We have found that pre-pretraining on formal languages can improve the language modeling loss and linguistic generalization abilities of transformer language models. In fact, pre-pretraining on some formal languages was more effective than increasing the amount of natural language training data by the same amount: the inductive bias conferred by the formal language was more helpful than additional in-distribution data. While most of the experiments in this paper were with 160M-parameter models, we also found benefits from pre-pretraining in the 1B-parameter setting.

We hypothesized that the languages that are most effective in this paradigm are those that, first, feature hierarchical dependencies, and second, are representable in $\CRASP$, and therefore readily learnable by transformers. The first part of this hypothesis is supposed by the superior performance of $k$-Dyck and $k$-Shuffle Dyck, the languages with hierarchical dependencies, relative to other languages. Our evidence for the importance of expressibility in $\CRASP$ is less clear: $k$-Shuffle Dyck clearly outperformed $k$-Dyck, which is not expressible in $\CRASP$, on one of the evaluations, and also required fewer steps of pre-pretraining than the other languages (\S \ref{sec:expressivity}). While natural language is arguably context-sensitive, in the Chomsky hierarchy sense, not every context-sensitive language was beneficial in pre-pretraining: in fact, pre-pretraining was harmful when we used the copy language $ww$, which, while context-sensitive since it contains cross-serial dependencies, does not illustrate the notion of hierarchy and is not definable in $\CRASP$. That being said, since $k$-Dyck performed almost as well as $k$-Shuffle Dyck, there may exist a sharper characterization of the class of formal languages that confer a helpful inductive bias than defined by our expressivity hypothesis; experiments with a larger sample of formal languages would be needed to progress towards such a characterization.


The marginal rate of substitution between formal and natural language is greater than one (Table~\ref{tab:ppt-results}), meaning that one token of formal language in pre-pretraining substitutes for more than one token of natural language in pretraining. This is a surprising result from the perspective of statistical learning theory \cite{Vapnik2000TheNO}, in that we observe faster convergence by swapping in data from a \textit{different distribution}. We hypothesize that initialization can have a critical effect on learning dynamics \cite{mccoy-etal-2020-berts,sellam2022the}, and pre-pretraining on formal language produces a initialization that is favorable to natural language learning.

\FloatBarrier

\section{Limitations}
\label{sec:limitations}

In this work, we considered blocked training, where we first train on formal language and then on natural language. While blocked training has the advantage that the initialization produced by formal language pre-pretraining can then be distributed and easily plugged into existing pretraining pipelines, it is possible that the optimal training regimen involves mixing formal and natural language during training \cite{tomek-mixing}. We also evaluated the effectiveness of pre-pretraining in a setting where natural language pretraining data is plentiful, as it is for English, such that it is possible to train the model for a considerable number of tokens without processing the same data multiple times over several epochs. We hypothesize that pre-pretraining will be even more effective for low-resource natural languages, and may yield different scaling properties with respect to pre-pretraining data \cite{Muennighoff2023ScalingDL}. 
Relatedly, a natural extension to this project is establishing scaling laws for pre-pretraining; the benefit of pre-pretraining beyond 1 billion parameters and 1.6 billion tokens is currently unknown. Finally, our work only considers transformers. Circuit complexity has also quantified the expressive power of neural networks like RNNs and state-space models \cite{merrill-etal-2020-formal,merrill2024the}, and it would be interesting to extend our results to these architectures.



\section*{Acknowledgments}

Many thanks to Andy Yang, Angelica Chen, Dan Friedman, Isabel Papadimitriou, Lindia Tjuatja, Mayee Chen, Qingyang Zhu, and the NYU Computation and Psycholinguistics Lab for feedback and discussion.
This work was supported in part through the NYU IT High Performance Computing resources, services, and staff expertise.
This project is supported by the National Science Foundation (NSF) under grant NRT-HDR: FUTURE as well as Grant No. IIS-2239862.
MYH is supported by the NSF Graduate Research Fellowship.
WM is supported by the NSF Graduate Research Fellowship as well as the Two Sigma PhD Fellowship.

\bibliography{anthology,custom}

\vfill
\pagebreak
\appendix

\section{Proofs}
\label{app:proofs}

We make use of the following to establish that all languages we consider are context-sensitive.

\begin{lemma} \label{lem:fom-csl}
    Any language definable in $\FOM$ can be recognized by a context-sensitive grammar.
\end{lemma}

\begin{proof}
    \citet{barrington1990uniformity} show that the class of languages definable in $\FOM$ is $\mathsf{LOGTIME}$-uniform $\mathsf{TC}^0$, which is a subset of $\mathsf L = \mathsf{SPACE}(\log n)$.
    On the other hand, the context-sensitive languages are those languages recognizable by linear-bounded automata \citep{kuroda1964classes}. That is, $\mathsf{CSL} = \mathsf{NSPACE}(n)$.
    Putting these characterizations together, we see that
    \begin{equation*}
        \mathsf{TC}^0 \subseteq \mathsf{SPACE}(\log n) \subseteq \mathsf{NSPACE}(n) = \mathsf{CSL} .
    \end{equation*}
    Therefore we can conclude that any language definable in $\FOM$ is context-sensitive.
\end{proof}

We will make use of the classical pumping lemma to establish that some specific languages considered are \emph{strictly} context-sensitive, i.e., not context-free.

\begin{lemma}[Pumping Lemma, \citealp{bar-hillel-1961-formal}]
    Let $L$ be a context-free language. Then there exists a pumping constant $p > 0$ such that any string $s \in L$ of length $\abs{s} \geq p$ can be written as $s = uvwxy$ where
    \begin{compactenum}
        \item $\abs{vwx} \leq p$;
        \item $\abs{vx} \geq 1$; and
        \item $uv^nwx^ny \in L$ for all $n \geq 0$.
    \end{compactenum}
\end{lemma}

Additionally, we will leverage the following communication complexity result to prove that certain languages are undefinable in $\CRASP$:

\begin{lemma}[\citealp{huang2025a}, Theorem 12] \label{lem:crasp-cc}
    Let $L$ be a language definable in $\CRASP$.
    Fix $w \in L$ and $1 \leq i \leq \lvert w \rvert$.
    Let Alice have access to $w_{1:i}$ and Bob have access to $w_{i+1:\lvert w \rvert}$.
    Then there exists a protocol where Alice sends at most $O(\log n)$ bits to Bob and Bob can recognize whether $w \in L$.
\end{lemma}

Crucially, if some $L$ requires Alice to send Bob $\omega(\log n)$ bits, then it cannot be defined in $\CRASP$.

We will also use the equivalence between \CRASP\ and the Temporal Counting Logic $\KtSharp$ to show that languages are definable in \CRASP.
\begin{lemma}[\citealp{yang2024counting}, Theorem 4.3] \label{lem:crasp-ktsharp}
    A \textsf{C-RASP} program recognizes language $L$ if and only if a $\KtSharp$ formula defines $L$.
\end{lemma}

\subsection{Language Characterizations}

\begin{proposition}
$1$-Dyck is context-free and definable in \textsf{C-RASP}.
\end{proposition}
\begin{proof}
That $1$-Dyck is context-free follows from the fact that it can be generated by the following context-free grammar:
\begin{align*}
    S &\to \texttt{(}\, S\, \texttt{)}\, S, \\
    S &\to \varepsilon.
\end{align*}
$1$-Dyck is defined by the following $\KtSharp$ formula 
$$
\left(\sinceC{Q_{(}}=\sinceC{Q_{)}}\right)\land \left(\sinceC{\sinceC{Q_{)}}>\sinceC{Q_{(}}}=0\right),
$$
and so is implementable in \textsf{C-RASP} by \Cref{lem:crasp-ktsharp}.
\end{proof}

\begin{proposition} \label{prop:kdyck}
    For $k \geq 2$, $k$-Dyck is context-free and not definable in $\CRASP$.
\end{proposition}
\begin{proof}
That $k$-Dyck is context-free follows from the fact that it can be generated by a context-free grammar: for any fixed value of $k$, $k$-Dyck is generated by 
\begin{align*}
S &\to \texttt{(}_{i} \,S\, \texttt{)}_{i}\, S, \quad\text{where}\; i \in [k] \\
S &\to \varepsilon
\end{align*}

To see that $k$-Dyck is not definable in $\CRASP$, consider Dyck strings of length $2n$ where tokens $1$ to $n$ are opening braces, and tokens $n+1$ to $2n$ are closing braces.
If Alice receives the first $n$ tokens, she must send $\Omega(n)$ bits to Bob if Bob is to correctly recognize the input string, because each prefix has a different unique suffix that closes it.
So $k$-Dyck is not in $\CRASP$ by Lemma \ref{lem:crasp-cc}.
\end{proof}

On the other hand, $k$-Dyck can be defined in $\FOM$.

\begin{proposition} \label{prop:kdyck-fom}
    For $k \geq 1$, $k$-Dyck is definable in $\FOM$.
\end{proposition}

\begin{proof}
Let $Q_((i)$ check whether token $i$ is \textit{any} of the $k$ opening parentheses, and $Q_{(_\kappa}(i)$ check whether token $i$ is the $\kappa$th opening parenthesis out of $k$. Continuing the definition from Section \ref{sec:circuit-background}:
\begin{align*}
    &\mathsf{depth}(i) \;\; \equiv \;\; \# j \leq i [Q_((i)] - \# j \leq i [Q_)(i)] \nonumber \\
    &\mathsf{dindex}(i) \;\; \equiv \;\; \#j \leq i [\mathsf{depth}(i) = \mathsf{depth}(j) ] \\
    &\mathsf{paired}(j, i) \;\; \equiv \;\; [\mathsf{depth}(j) = \mathsf{depth}(i) + 1] \wedge \\
    &\quad [\mathsf{dindex}(i) = \mathsf{dindex}(j)] \\
    &\mathsf{match}(j, i) \;\; \equiv \;\; \bigvee_\kappa  [Q_{(_\kappa}(j) \wedge Q_{)_\kappa}(i)] \\
    &\mathsf{closed}(i)  \;\; \equiv \;\; \exists j \leq i \left[ \mathsf{paired}(j, i) \wedge \mathsf{match}(j, i) \right]
\end{align*}

Having defined these macros, we are now ready to write the recognizer for $k$-Dyck:
\begin{align*}
    &[\mathsf{depth}(n) = 0] \wedge [\# i \leq n [ \mathsf{depth}(i) < 0] = 0]  \wedge \\
    &\quad \forall i \leq n \left[ \mathsf{closed}(i) \right]
\end{align*}
\end{proof}

To understand why this construction cannot be implemented in $\CRASP$,
observe that $\mathsf{paired}(j, i)$ and $\mathsf{match}(j, i)$ are binary predicates, which are not allowed in $\CRASP$.

\begin{lemma}
    For $k \geq 2$, $k$-Shuffle Dyck is strictly context-sensitive and definable in \textsf{C-RASP}.
\end{lemma}
\begin{proof}
\emph{See Ex.\ 7.20 in \citet{hopcroft-2006-introduction}}.
Consider the case when $k=2$.
Assume that $2$-shuffle Dyck is context-free.
Then $L = \texttt{(}^n \texttt{[}^m \texttt{)}^n \texttt{]}^m$ is context-free since it is the intersection of $k$-Shuffle Dyck with $\texttt{(}^*\texttt{[}^*\texttt{)}^*\texttt{]}^*$ and CFLs are closed under intersection with regular languages.

Assume by way of contradiction that $L$ is context-free and so has pumping constant $p$. Let $s = \texttt{(}^p \texttt{[}^p \texttt{)}^p \texttt{]}^p$, whichy by hypothesis can be written as $uvwxy$. Since $\abs{vwx} \leq p$, it either (a) lies entirely inside one of the blocks of $p$ symbols or (b) lies partially in one block of $p$ symbols and lies partially in at most one adjacent block. In the case of (a), suppose for clarity that $vwx$ lies entirely in the $\texttt{(}^p$ block. Since ${vx}$ is not empty, $uv^0wx^0y \equiv uwy$ contains fewer $\texttt{(}$'s than $\texttt{)}$'s, and hence is not in $L$, a contradiction. In the case of (b), suppose for clarity that $vwx$ straddles the $\texttt{(}^p$ and $\texttt{[}^p$ blocks. Since ${vx}$ is not empty, $uv^0wx^0y \equiv uwy$ contains either fewer $\texttt{(}$'s than $\texttt{)}$'s or fewer $\texttt{[}$'s than $\texttt{]}$'s, and hence is not in $L$, a contradiction. Since $k$-Shuffle Dyck for $k > 2$ contains $2$-Shuffle Dyck, proving the $k=2$ case is sufficient to establish that $k$-Shuffle Dyck is not context free (but still context-sensitive by \Cref{lem:fom-csl}).


Similar to the $1$-Dyck case, we can exhibit a $\KtSharp$ formula to define $k$-Shuffle Dyck:
$$
\bigwedge_\kappa \left(\sinceC{Q_{(_\kappa}}=\sinceC{Q_{)_\kappa}}\right)\land \left(\sinceC{\sinceC{Q_{)_\kappa}}>\sinceC{Q_{(_\kappa}}}=0\right)
$$
So $k$-Shuffle Dyck is likewise definable in \textsf{C-RASP} by \Cref{lem:crasp-ktsharp}.
\end{proof}

\begin{proposition}
    $ww$ is strictly context-sensitive and not definable in $\CRASP$.
\end{proposition}
\begin{proof}
    \emph{See Ex.\ 7.21 in \citet{hopcroft-2006-introduction}}.
    Suppose by way of contradiction that $ww$ is context-free, and so has a pumping constant $p$. Let $s = \texttt{a}^p\texttt{b}^p\texttt{a}^p\texttt{b}^p$, which can be written as $uvwxy$ by hypothesis. Consider then the string $uv^0wx^0y \equiv uwy \in L$. We examine two cases depending on the position of $vwx$ in $s$. 
    
    In the first case, suppose $vx$ is contained entirely within the first block of $\texttt{a}^p$. If $\abs{vx} = k$ then $uwy$ has length $4p-k$ and begins with the substring $\texttt{a}^{(p-k)}\texttt{b}^p...$ of length $2p-k$. By assumption $uwy = tt$ for some $t$ of length $2p-k/2$, and since $k \geq 1$ it follows that $\abs{t} > 2p-k$. Then the final symbol of $t$ must lie within the \emph{second} block of $\texttt{a}$'s; yet since $s$ ends in $\texttt{b}$, $tt$ must also end in $\texttt{b}$, a contradiction.

    In the second case, suppose $vx$ contains some $\texttt{a}$'s and some $\texttt{b}$'s. Since $\abs{uvwxy} = 4p$ and $\abs{vwx} \leq p$ it must be that $\abs{uwy} \geq 3p$ and so $\abs{t} = 3p/2$. Since $vwx$ is too short to straddle more than two adjacent blocks of symbols and $3p/2 > p$ it must be the case that $t$ must end in $\texttt{b}^p$. Yet there since $\abs{vx} \geq 1$, there is only a single block of $\texttt{b}^p$ within $\abs{uwy}$, so the $\texttt{b}^p$ block cannot be repeated, a contradiction.

    By symmetry, these two cases straightforwardly extend to the cases when $vx$ is contained entirely within the first block of $\texttt{b}$'s, the second block of $\texttt{a}$'s, or the second block of $\texttt{b}$'s (analogous to case 1); or when it is split between the blocks of $\texttt{a}$'s and $\texttt{b}$'s (case 2). Then $ww$ is not context-free, but still context-sensitive by \Cref{lem:fom-csl}.


    From a communication complexity perspective, if Alice has the first half of some string, and Bob has the second half, Alice must send Bob $\Omega(n)$ bits to verify whether the string is of the form $ww$.
    Thus, by \Cref{lem:crasp-cc}, $ww$ cannot be defined in $\CRASP$.
\end{proof}

\begin{figure}[htbp] 
    \centering 
    \includegraphics[width=0.95\columnwidth]{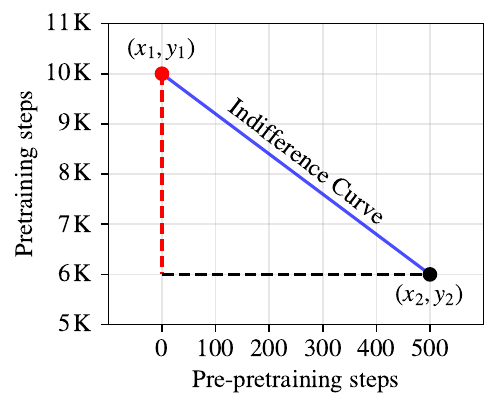} 
    \caption{The indifference curve contains points with equal training loss. Marginal rate of substitution is the ratio between the red and black lines $\left( \frac{|y_1-y_2|}{|x_1-x_2|} \right)$. The token efficiency increase from applying pre-pretraining can be calculated as $1 - \frac{x_2  + y_2}{y_1}$.} 
    \label{fig:mrs} 
\end{figure}

\section{Hyperparameters}
\label{app:hyperparams}

All experiments were done on NVIDIA A100 or H100 80GB GPUs. We warm up the learning rate both during pre-pretraining and pretraining. The below hyperparameters hold for both pre-pretraining and pretraining. That is, for simplicity, even if we only pre-pretrain for 500 steps, we still keep the learning rate warmup at 1,000 steps. 
To achieve 50\% attention head sparsity when pruning, we set the target sparsity to 70\%. We used Huggingface \texttt{transformers==4.47.0} and \texttt{datasets==3.2.0}.

\begin{table}[h]    

\begin{tabular}{ll}
\toprule
\textbf{Hyperparameter} & \textbf{Value} \\
\midrule
\multicolumn{2}{l}{\textit{Training Configuration}} \\
\midrule
Batch size & 16 \\
Gradient accumulation & 2 \\
Effective bsz & 32 \\
Sequence length & 2048 tokens \\
Learning rate & $5 \times 10^{-4}$ \\
LR schedule & Cosine w/ warmup \\
Min. LR & $5 \times 10^{-5}$ \\
Warmup Steps & 1000 \\
Weight Decay & 0.1 \\
Gradient Clipping & 1.0 \\
\midrule
\multicolumn{2}{l}{\textit{Optimization}} \\
\midrule
Optimizer & AdamW \\
$\beta_1, \beta_2$ & 0.9, 0.999 \\
$\epsilon$ & 1e-6 \\
Mixed Precision & bf16 \\
\midrule
\multicolumn{2}{l}{\textit{Pruning} (see \citet{bhaskar-etal-2024-heuristic})} \\
\midrule
Learning rate & 0.1 \\
Regularization LR & 1 \\
Target sparsity & 0.7 \\
Warmup steps & 1000 \\
\bottomrule
\end{tabular}
\caption{Training hyperparameters.}
\label{tab:lm-hyperparams}
\end{table}

\newcommand{\posex}[1]{\textcolor{green!50!black}{\ding{51} #1}}
\newcommand{\negex}[1]{\textcolor{black}{\ding{55} #1}}

\begin{table*}[htbp]
\centering
\begin{tabular}{>{\raggedright\arraybackslash}p{0.45\textwidth}>{\raggedright\arraybackslash}p{0.45\textwidth}}
\toprule
\textbf{Positive Example} & \textbf{Negative Example} \\
\midrule
\posex{Only Bill would ever complain.} & \negex{Even Bill would ever complain.} \\
\posex{Diane watched Alan.} & \negex{Diane screamed Alan.} \\
\posex{Who should Derek hug after shocking Richard?} & \negex{Who should Derek hug Richard after shocking?} \\
\bottomrule
\end{tabular}
\caption{Examples from the BLiMP dataset \cite{warstadt-etal-2020-blimp-benchmark}: matched pairs of grammatical (positive, left) and ungrammatical (negative, right) sentences. We expect the language model to assign a higher probability to the grammatical sentence in each pair.}
\label{tab:blimp}
\end{table*}

\begin{table*}[htbp]
\centering
\renewcommand{\arraystretch}{1.3} 
\begin{tabular}{p{0.8\textwidth}}
\toprule
\textbf{Examples} \\
\midrule
Before the meeting, Mary wrote down the following list of words:  
\textbf{window, door, roof}. After the meeting, she took a break and had a cup of coffee. When she got back, she read the list again:  
\textbf{window, door, roof}. \\
\midrule
Before the meeting, John wrote down the following list of words: \textbf{nothing, riches, paper}. After the meeting, he took a break and had a cup of coffee. When he got back, he read the list again: \textbf{nothing, riches, paper}. \\
\bottomrule
\end{tabular}
\caption{Verbatim in-context retrieval \cite{armeni-etal-2022-characterizing,armeni-etal-2024-transformer} examples. We expect a good language model to recognize based on the context that the list is repeated, retrieve the appropriate items from the first repetition of the list, and assign these items a very high probability.}
\label{tab:verbatim}
\end{table*}

\begin{figure*}[htbp] 
    \centering 
    \includegraphics[width=\textwidth]{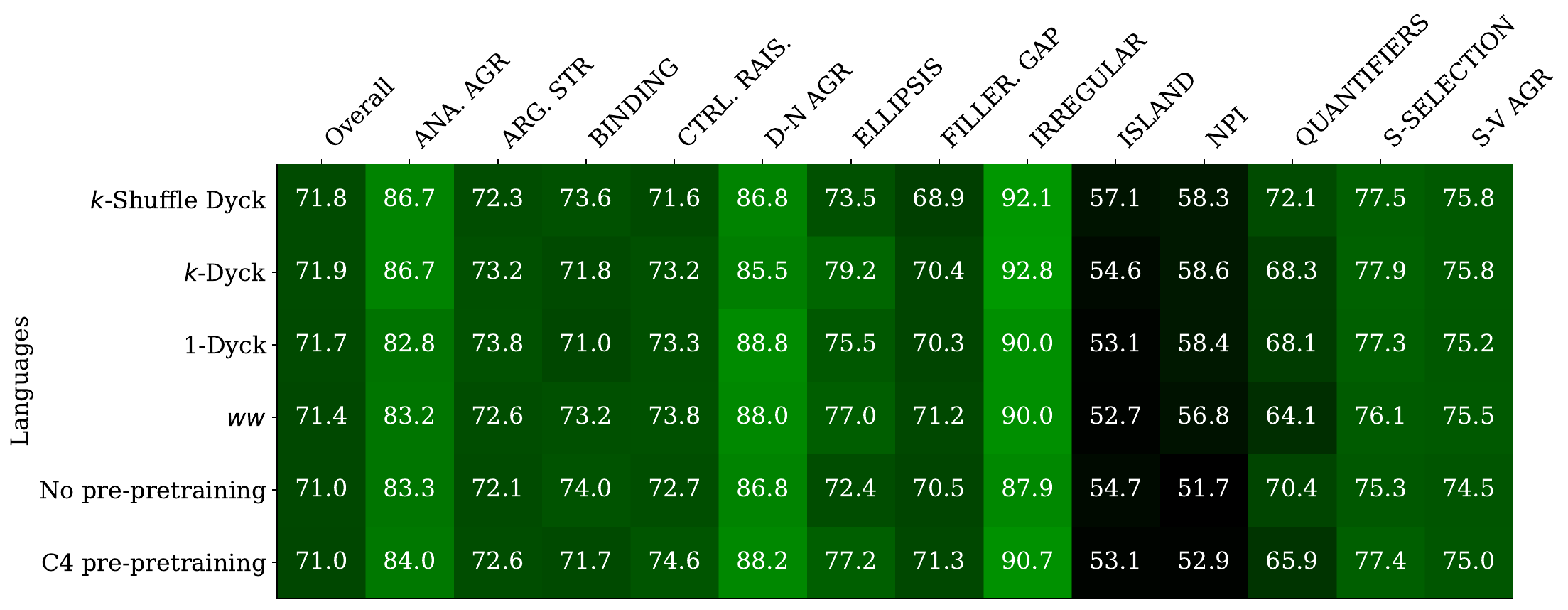} 
    \caption{Accuracy on BLiMP by grammatical phenomenon. The full names of the phenomena are: anaphor agreement, argument structure, binding, control/raising, determiner-noun agreement, ellipsis, filler-gap dependencies, irregular forms, island effects, negative polarity item licensing, quantifiers, and subject-verb agreement.} 
    \label{fig:blimp-accuracies} 
\end{figure*}

\begin{figure*}[htbp] 
    \centering 
    \includegraphics[width=\textwidth]{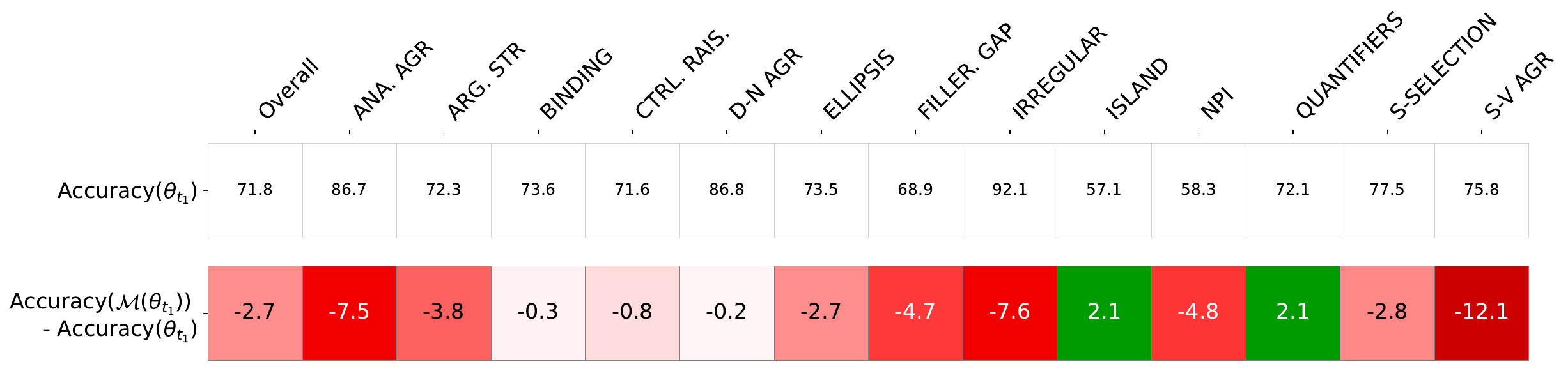} 
    \caption{Performance changes on BLiMP after pruning half the attention heads from the model trained on $k$-Shuffle Dyck (see \S\ref{sec:pruning}). The largest declines are on subject-verb agreement, irregular forms, and anaphor agreement. These categories require knowledge about word forms, and the sentences within these categories are generally simple (around 4 words).} 
    \label{fig:blimp-mask-delta} 
\end{figure*}

\begin{table*}[ht]
\centering

\begin{tabularx}{\textwidth}{ll cccc c}
\toprule
& \textbf{Language} & \textbf{LM Loss} & \textbf{Documents $\downarrow$} & \textbf{MRS} & \textbf{Grammaticality} & \textbf{Retrieval} \\ 
\midrule
\multirow{5}{*}{Formal} & 1-Dyck & {3.760 {\scriptsize{$\pm$0.016}}} & {0.978 {\scriptsize{$\pm$0.021}}} & 3.01 & {0.717 {\scriptsize{$\pm$0.004}}} & {3.373 {\scriptsize{$\pm$0.014}}} \\ 
& $k$-Dyck & {3.743 {\scriptsize{$\pm$0.016}}} & \bfseries{0.998 {\scriptsize{$\pm$0.001}}} & 3.57 & \bfseries{0.719 {\scriptsize{$\pm$0.003}}} & {3.338 {\scriptsize{$\pm$0.019}}} \\ 
& $k$-Shuffle Dyck & \bfseries{3.741 {\scriptsize{$\pm$0.014}}} & \bfseries{0.998 {\scriptsize{$\pm$0.001}}} & \bfseries{7.15} & {0.718 {\scriptsize{$\pm$0.007}}} & \bfseries{3.297 {\scriptsize{$\pm$0.012}}} \\
& $ww$ & {3.792 {\scriptsize{$\pm$0.018}}} & {0.557 {\scriptsize{$\pm$0.247}}} & {$-$}0.25 & {0.714 {\scriptsize{$\pm$0.003}}} & {3.341 {\scriptsize{$\pm$0.021}}} \\
\midrule
\multirow{4}{*}{Controls} & No pre-pretraining & {3.780 {\scriptsize{$\pm$0.018}}} & {---} & {---} & {0.710 {\scriptsize{$\pm$0.011}}} & {3.393 {\scriptsize{$\pm$0.003}}} \\ 
& C4 pre-pretraining & {3.754 {\scriptsize{$\pm$0.017}}} & {0.992 {\scriptsize{$\pm$0.007}}} & 6.65 & {0.710 {\scriptsize{$\pm$0.003}}} & {3.354 {\scriptsize{$\pm$0.005}}} \\ 
& Random binary & {3.810 {\scriptsize{$\pm$0.015}}} & {0 {\scriptsize{$\pm$0}}} & {$-$}6.60 & {0.712 {\scriptsize{$\pm$0.004}}} & {3.416 {\scriptsize{$\pm$0.016}}} \\ 
& Random ints & {3.798 {\scriptsize{$\pm$0.015}}} & {0.042 {\scriptsize{$\pm$0.041}}} & {$-$}5.97 & {0.712 {\scriptsize{$\pm$0.006}}} & {3.409 {\scriptsize{$\pm$0.006}}} \\ 
\bottomrule
\end{tabularx}

\caption{Evaluating models at the optimal amount of pre-pretraining $t_0^*$ for each formal language  (see \S\ref{sec:expressivity}). ``Documents $\downarrow$'' is the proportion of documents in the C4 validation set where the model has a lower loss than the model trained without pre-pretraining. 1-Dyck, $k$-Dyck, and $k$-Shuffle-Dyck all have marginal rates of substitution (MRS) greater than~1, indicating that pre-pretraining is more efficient than not pre-pretraining. $k$-Shuffle-Dyck performs the best overall on our evaluation metrics.}
\label{tab:ppt-results}
\end{table*}

\begin{figure*}[t]
\renewcommand{\lstlistingname}{Code} 
\caption{Implementation of a $k$-Shuffle Dyck sequence generator.}
\label{code:shuff-dyck}
\begin{lstlisting}[language=Python]
import random


def generate_shuff_dyck(k, max_length=2048, p_open=0.5, max_depth=16):
    """
    Generate a k-shuffle Dyck sequence, truncated at max_length.
    When max depth is reached, close one bracket and continue.

    Args:
        k (int): Number of different types of brackets
        max_length (int): Target maximum length of the sequence
        p_open (float): Probability of opening a new bracket
        max_depth (int): Maximum nesting depth allowed

    Returns:
        list: Generated sequence where i represents opening bracket i
             and i+k represents closing bracket i

    Note: the final Dyck word may be invalid due to truncation, but
    we didn't find this to be an issue in practice.
    """
    sequence = []
    counts = [0] * k  # Track open brackets of each type

    while len(sequence) < max_length:
        depth = sum(counts)

        # Must open if all brackets are closed
        if depth == 0:
            bracket = random.randint(0, k - 1)
            sequence.append(bracket)
            counts[bracket] += 1
            continue

        # If at max depth, force a close
        if depth >= max_depth:
            open_brackets = [i for i, count in enumerate(counts) if count > 0]
            bracket = random.choice(open_brackets)
            sequence.append(bracket + k)
            counts[bracket] -= 1
            continue

        # Randomly choose to open or close
        if random.random() < p_open and depth < max_depth:
            bracket = random.randint(0, k - 1)
            sequence.append(bracket)
            counts[bracket] += 1
        else:
            # Close an existing bracket
            open_brackets = [i for i, count in enumerate(counts) if count > 0]
            bracket = random.choice(open_brackets)
            sequence.append(bracket + k)
            counts[bracket] -= 1

    return sequence

\end{lstlisting}
\end{figure*}

\end{document}